\documentclass[lettersize,journal]{IEEEtran}
\usepackage{algorithm}
\usepackage{algorithmic}

\usepackage{amsthm}
\usepackage{amsmath,amssymb,mathtools}
\usepackage{makecell}
\usepackage{booktabs}
\usepackage{footnote}
\usepackage{cite}
\usepackage[switch]{lineno}  %

\newtheorem{theorem}{Theorem}

\newtheorem{lemma}{Lemma}
\newtheorem{corollary}{Corollary}
\newtheorem{assumption}{Assumption}

\newtheorem{remark}{Remark}

\usepackage{ifpdf}
\usepackage[flushleft]{threeparttable}
\usepackage{makecell}
\usepackage{multirow}
\usepackage{bbding}
\usepackage{graphicx}
\usepackage{dsfont}
\newcommand{\tool}{\texttt{SL-SAM}}
\usepackage{pifont}
\newcommand{\myparagraph}[1]{\textbf{#1}\hspace{0.3em}}
\usepackage[most]{tcolorbox}
\usepackage[table]{xcolor}

\ifCLASSINFOpdf
\else
\fi
\begin{document}
\title{Sparse Layer Sharpness-Aware Minimization for Efficient Fine-Tuning}

\author{Yifei~Cheng, Xianglin~Yang, Guoxia~Wang, Chao~Huang, Fei~Ma, Dianhai~Yu, Xiaochun Cao, Li Shen

\IEEEcompsocitemizethanks{
\IEEEcompsocthanksitem Y. Cheng, C. Huang, X. Cao and L. Shen are with the School of Cyber Science and Technology, Sun Yat-sen University, Shenzhen Campus. E-mail: \{yfcheng.ifc, mathshenli\}@gmail.com, \{huangch253, caoxiaochun\}@mail.sysu.edu.cn. \protect
\IEEEcompsocthanksitem X. Yang is with the School of Computing, National University of Singapore. E-mail: yxlyzxl@gmail.com. \protect
\IEEEcompsocthanksitem G. Wang and D. Yu are with the Baidu Inc. E-mail: \{wangguoxia, yudianhai\}@baidu.com. \protect
\IEEEcompsocthanksitem F. Ma is with the Guangdong Laboratory of Artificial Intelligence and Digital Economy (SZ). E-mail: mafei@gml.ac.cn. \protect
\IEEEcompsocthanksitem Corresponding author: Li Shen.}% <-this % stops a space
}


\markboth{Journal of \LaTeX\ Class Files,~Vol.~14, No.~8, August~2015}%
{Shell \MakeLowercase{\textit{et al.}}: Bare Advanced Demo of IEEEtran.cls for IEEE Computer Society Journals}



\IEEEtitleabstractindextext{%

\begin{abstract}
Sharpness-aware minimization (SAM) seeks the minima with a flat loss landscape to improve the generalization performance in machine learning tasks, including fine-tuning. However, its extra parameter perturbation step doubles the computation cost, which becomes the bottleneck of SAM in the practical implementation. In this work, we propose an approach \tool\ to break this bottleneck by introducing the sparse technique to layers. Our key innovation is to frame the dynamic selection of layers for both the gradient ascent (perturbation) and descent (update) steps as a multi-armed bandit problem. At the beginning of each iteration, \tool\ samples a part of the layers of the model according to the gradient norm to participate in the backpropagation of the following parameter perturbation and update steps, thereby reducing the computation complexity. We then provide the analysis to guarantee the convergence of \tool. In the experiments of fine-tuning models in several tasks, \tool\ achieves the performances comparable to the state-of-the-art baselines, including a \#1 rank on LLM fine-tuning. Meanwhile, \tool\ significantly reduces the ratio of active parameters in backpropagation compared to vanilla SAM (\tool\ activates 47\%, 22\% and 21\% parameters on the vision, moderate and large language model respectively while vanilla SAM always activates 100\%), verifying the efficiency of our proposed algorithm.
\end{abstract}

}

\maketitle

% is not selected <OR> if conference mode is selected - because compsoc
% conference papers position the abstract like regular (non-compsoc)
\IEEEdisplaynontitleabstractindextext

% \fi
\IEEEpeerreviewmaketitle

\section{Introduction}
\IEEEPARstart{G}{eneralization} ability is an important metric in machine learning as the model is susceptible to being over-trained on the training set and may face distribution shifts between the training and test data. The studies in \cite{keskar2016large,neyshabur2017exploring} reveal the connection between the generalization ability and sharpness of the loss landscape. Based on this, Sharpness-Aware Minimization (SAM) is proposed and renowned as a powerful optimization technique for its ability to enhance model generalization by seeking \textit{flat} loss minima \cite{foret2020sharpness}. Its mechanism operates in a two-step process: it first computes a worst-case adversarial perturbation that maximizes loss within a local neighborhood, and then updates the model parameters by minimizing the loss at this perturbed point. This approach has been highly successful, further incurring widespread adoption in resource-intensive tasks such as fine-tuning large-scale Transformers models and Large Language Models (LLMs) with billion-level parameters\cite{zhong2022improving,sun2024adasam,li2024revisiting}.

\begin{figure}[t]
\centering
\includegraphics[width=0.48\textwidth]{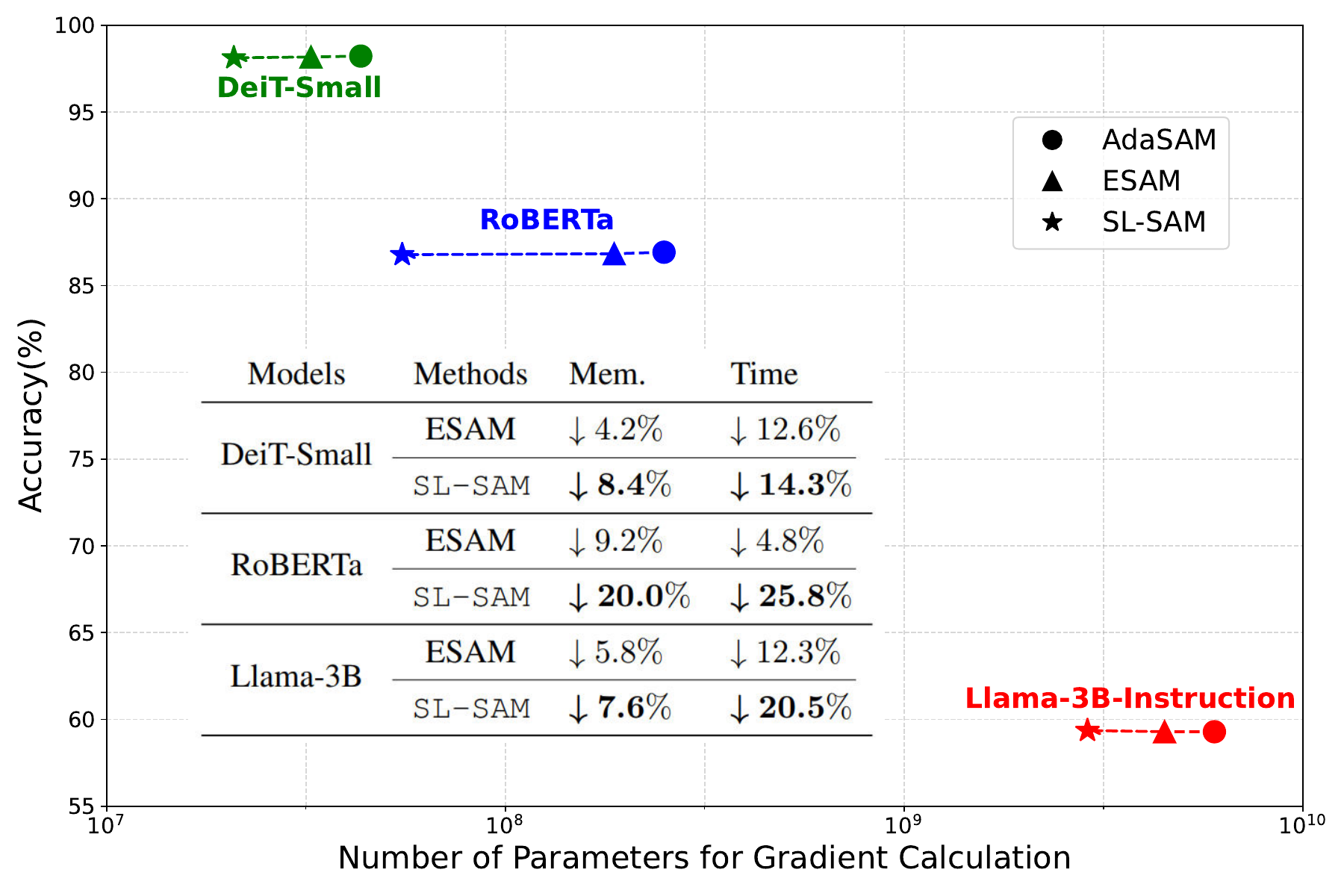}
\caption{\textbf{Average number of model parameters that participate in gradient calculation per iteration vs. task accuracy.} We compare \tool\ with AdaSAM \cite{sun2024adasam} (vanilla SAM with AdamW) and a representative efficient variant ESAM \cite{du2021efficient}. The points connected by dotted lines indicate that they achieve comparable performances, while the parameters participate in backpropagation in \tool\ is 47\%, 22\% and 21\% of that in AdaSAM across three tasks. The results in the table show the GPU memory and wall-clock time savings compared to AdaSAM. Details are referred to the Experiments section.}
\label{overview}
\end{figure}

However, this performance gain comes at a steep price: by requiring two sequential forward and backward passes (perturbation and update), SAM effectively doubles the gradient calculation time of classical optimizers. This prohibitive computational overhead severely limits its practical adoption, creating a critical bottleneck in real-world applications. Naturally, we raise a question:
\begin{center}
    \textit{Can we improve the computation efficiency of SAM?}
\end{center}

\begin{figure*}[t]
\centering
\includegraphics[width=0.95\textwidth]{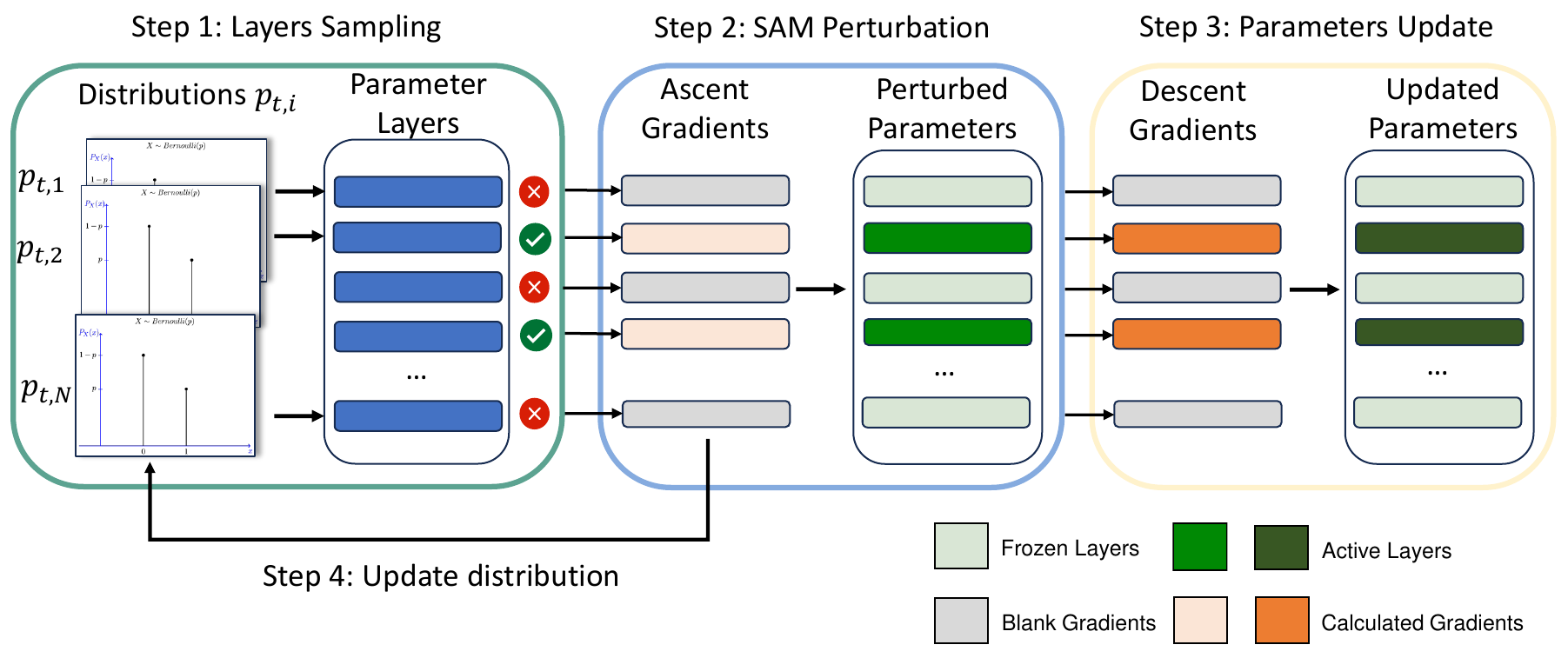}
\caption{\textbf{The workflow of our algorithm \tool.} Step 1: all the parameter blocks (in blue) are sampled according to the distributions: the active layers participate in this training iteration while the others are frozen; Step 2 \& 3: calculate the gradients for active layers to perform the parameter perturbation and update steps; Step 4: update the distributions by the gradient norms obtained in the perturbation.}
\label{fig1}
\end{figure*}

While prior works have sought to improve SAM's efficiency, they often focus on reducing the frequency of the perturbation step \cite{zhao2022randomized,liu2022towards} or introducing sparsity in a suboptimal manner. For instance, ESAM \cite{du2021efficient} applies perturbation to a randomly selected subset of parameters, failing to distinguish which are most critical to stabilize. SSAMF \cite{mi2022make} considers the adoption of the sparse mask, but still needing to obtain the exact gradient first. This situation leaves a crucial gap: proposing a method that indeed reduces the gradient computation overhead in both of SAM's two steps.

To address this gap, we introduce \textbf{Sparse-Layer SAM} (\tool), a novel framework that injects the adaptive layerwise sparsity into SAM. Our approach is predicated on the insight that the gradient norm of the model layer is a strong proxy for a layer's importance to both the perturbation and update steps, motivating us to activate only partial layers' parameters to participate in the gradient calculation. However, a naive greedy strategy for conducting the sparse gradient that simply selects the layers with the top gradient norms at each step is brittle. The set of important layers is highly dynamic, and a greedy approach is susceptible to the stochasticity of mini-batch gradients, failing to explore layers that are temporarily dormant but may become critical for convergence.
Therefore, our key innovation is to reframe layer selection as a multi-armed bandit problem. This enables \tool\ to robustly balance exploration and exploitation, dynamically identifying the most salient layers for optimization over time. In each iteration, the bandit samples an active set of layers for SAM's process: a sparse gradient ascent for perturbation, followed by a sparse gradient descent for model update. All other layers are frozen to skip the gradient calculation to save computation in this iteration. The bandit is guided by using each layer's gradient norm as a reward signal, ensuring computational resources are intelligently focused on the most critical layers of the model. This process is visualized in Figure \ref{fig1}.

We evaluate our method with both theoretical and empirical validation. We formally prove that \tool\ achieves a convergence rate of $\frac{1}{T}\sum_{t=1}^T \mathbb{E}\|\nabla f(x_t)\|_1 = \mathcal{O}(T^{-1/4})$, guaranteeing its optimization property. Furthermore, extensive fine-tuning experiments on DeiT, RoBERTa, and large language models demonstrate that \tool\ delivers a state-of-the-art trade-off between the efficiency and performance. It consistently achieves top-tier performances, including a \#1 rank in downstream evaluation tasks for the LLM, while reducing training time by up to 25.8\% and significantly cutting memory costs compared to vanilla SAM.

Our main contributions are as follows:
\begin{itemize}
    \item \textbf{A Novel Two-Sided Sparsity Framework:} We propose \tool, an optimization algorithm that applies the layerwise sparsity to both the perturbation and update steps of SAM to improve its efficiency. To our best knowledge, this is the first work to achieve this sparsification effect.
    \item \textbf{Formal Convergence Guarantee:} We theoretically prove that \tool\ achieves a convergence rate of $\frac{1}{T}\sum_{t=1}^T \mathbb{E}\|\nabla f(x_t)\|_1 = \mathcal{O}(T^{-1/4})$, providing the first theoretical evidence that a sparse SAM method can maintain the rate of the original SAM.
    \item \textbf{Superior Efficiency and Competitive Performance:} We demonstrate that \tool\ offers the best efficiency while maintaining the great performance. Across several fine-tuning tasks, it always achieves scores comparable to other baselines, especially the highest score on LLM fine-tuning. Benefiting from the sparse gradient, the average number of parameters involved in gradient calculation of \tool\ is dramatically reduced, resulting in significant savings in GPU memory and training time. We show an overview in Figure~\ref{overview}.
    \item \textbf{Great Compatibility:} We extend to inject the sparse layerwise gradient computation into the single-step version of SAM proposed recently. The experimental results still show an efficiency improvement with a marginal loss of performance. Thus, our strategy is compatible with different goals: whether the great performance achieved by the double-step SAM or the shorter training time realized by the single-step SAM.
\end{itemize}

\section{Related Work}

\subsection{Sharpness-Aware Minimization}
Through solving a minimax optimization problem considering the worst-case in the neighborhood, Sharpness-Aware Minimization (SAM)~\cite{foret2020sharpness} finds a flat loss minima by parameter perturbation and update steps. Its success in improving the model generalization has led to a series of works aiming to further enhance its performance \cite{kwon2021asam,mi2022make,liu2022random,zhuang2022surrogate,li2024friendly} and spurred theoretical analysis of its convergence or revelation of its mechanisms~\cite{andriushchenko2022towards, wen2022sharpness, dai2023crucial}. 

Sufficient experiments have shown that the great generalization ability of SAM is also applicable to language models \cite{bahri2021sharpness}.
Among the following works on language models, AdaSAM \cite{sun2024adasam} replaces the base optimizer with AMSGrad together with a convergence analysis. FSAM \cite{zhong2022improving} further suggests generating the Fisher mask to only perturb elements with large gradient values. Recently, Li et al. \cite{li2024revisiting} observe that incorporating SAM into AdamW could significantly mitigate catastrophic forgetting in continual tuning LLMs. However, SAM's primary drawback is its computational cost, as its two-step process doubles the gradient calculation time.

Efforts to mitigate this cost often have some limitations. RST~\cite{zhao2022randomized} conducts the perturbation step randomly; LookSAM \cite{liu2022towards} and SALA \cite{tan2024sharpness} select the strategy that performs SAM periodically. Although these methods could reduce the number of perturbation steps, they still need to calculate the full-dimensional gradients, resulting in the same memory cost as SAM. ESAM~\cite{du2021efficient} and SSAMF~\cite{mi2022make} adopt the heuristic sparsify strategy: the former randomly samples parameters to perturb and the latter employs a sparse mask on the perturbation. However, they also involve the full gradient calculation in two steps and introduce extra steps as selecting a subset of samples and obtaining the sparse mask respectively. SAMPa~\cite{xie2024sampa} conducts the parallelism for the two steps which increases memory usage. In addition, most of them are not verified in the fine-tuning tasks.

Thus, sparsing \textit{both} of SAM's core steps to make the fine-tuning efficient remains an open challenge.

\subsection{Principled Sparsity for Efficient Training}
Sparsity is a cornerstone of efficient deep learning, with representative techniques such as Dropout~\cite{srivastava2014dropout}, and network pruning~\cite{lecun1989optimal,hassibi1993optimal,gale2019state}. The huge size of deep learning models poses the demand for cost reduction, which matches the effects of the sparse structure. Sun et al. \cite{sun2023simple} introduce pruning into the large-size model to reduce the computational cost. DropIT \cite{chen2022dropit} and Back Razor \cite{jiang2022back} drop the activations' elements with small absolute values in the forward pass to reduce the memory cost of activations, then recover the sparse tensors to be dense to compute an approximated gradient in backpropagation. JointSpar \cite{liu2022communication} samples partial layers to calculate the gradients and freezes the others in backpropagation to accelerate the local computation and communication in distributed learning. A key innovation in this method is that the layer selection is framed as a multi-armed bandit problem to adaptively allocate computational resources. PreBackRazor \cite{yu2024sheared} inherits this technique to sample the activations to be pruned and skip the gradient computation for them to improve Back Razor, achieving both memory and computation efficient. Our work is the first to adapt this technique to solve the two-step optimization problem of SAM. We employ the bandit formulation that explicitly accounts for both the gradient ascent and descent steps, creating a method uniquely tailored to the dynamics of sharpness-aware optimization.

\section{Methodology}

This section develops our proposed method \tool. We first review the fundamentals of Sharpness-Aware Minimization (SAM) to establish the context for our work. We then present the motivation for our algorithm, which leverages the gradient norm as a unified criterion for sparsity. Following this, we formally introduce the \tool\ algorithm and conclude with a theoretical analysis that guarantees its convergence.
\subsection{Preliminary: Sharpness-Aware Minimization}
We address the standard non-convex optimization problem in deep learning:
\begin{equation}
    \min_{x \in \mathbb{R}^d} f(x) := \mathbb{E}_{\xi \sim \mathcal{D}} [f(x, \xi)],
\end{equation}
where $x \in \mathbb{R}^d$ represents the model parameters and $f(x, \xi)$ is the loss for a sample $\xi$ from the data distribution $\mathcal{D}$. While standard optimizers can converge to ``sharp'' minima that generalize poorly, Sharpness-Aware Minimization (SAM) \cite{foret2020sharpness} seeks ``flat'' minima to improve generalization.

SAM achieves this via a two-step process at each iteration $t$. Firstly, it finds an adversarial perturbation $\epsilon_t(x_t)$ to perform a gradient ascent step to maximize the loss within a $\rho$-ball neighborhood:
\begin{equation}
    \epsilon_t(x_t) = \rho \frac{\nabla f(x_t, \xi_t)}{\|\nabla f(x_t, \xi_t)\|_2}.
\end{equation}
Secondly, it computes the gradient at this perturbed point, then performs gradient descent to update the parameters.
\begin{equation}
    x_{t+1} = x_t - \eta \nabla f(x_t + \epsilon_t(x_t), \xi_t).
\end{equation}
This two-step process requires two full-dimensional forward and backward passes, making the computational cost of SAM higher than standard optimizers. This bottleneck is the primary challenge we aim to address.

\subsection{Motivation}
Our approach is motivated by the insight that not all layers contribute equally at each training step. We hypothesize that the layer-wise gradient norm is a powerful and unified proxy for identifying the most critical layers for \textit{both} of SAM's steps. This is supported by two established lines of research:

\begin{enumerate}
    \item \textbf{For the Perturbation Step:} 
    Generally, the landscape flatnesses of the model parameters are not uniformly distributed. We could give priority to the parameters with bad flatness to perform perturbation, as a result, only calculating the sparse gradients with respect to these parameters. Zhao et al. \cite{zhao2022penalizing} point out that penalizing the gradient norm leads the loss function to have a small Lipschitz constant, further indicating the model has a flat landscape. This implies that layers with a \textbf{larger gradient norm} likely reside in sharper regions and are the highest priority for SAM's perturbation.
    
    \item \textbf{For the Update Step:} The goal of the descent step is to optimize the model. It is well-established that parameters with a \textbf{larger gradient norm} are most influential in driving the optimization process \cite{johnson2018training, sung2021training}.
\end{enumerate}

These two points converge on one powerful principle:
\begin{tcolorbox}[
  enhanced,
  breakable,
  colback=gray!12,   
  colframe=black,    
  boxrule=0.8pt,     
  arc=1.5mm,         
  left=2mm,right=2mm,top=1mm,bottom=1mm
]
\begin{center}
  \textit{The parameters with a large gradient norm are the most critical for both generalization (perturbation) and convergence (update).}
\end{center}
\end{tcolorbox}

While this principle suggests a simple strategy, i.e. deterministically selecting the layers with the top-$k$ largest gradient norms in each iteration. A recent work SAFER \cite{gopal2025safer} just employs this strategy to sample layers with large norms of gradients of the adversarial loss, then applies SAM on them to address the adversarial overfitting. However, such a greedy strategy is suboptimal for two critical reasons. Firstly, it is purely \textbf{exploitative}, risking that layers with temporarily small gradients are never selected and thus never updated, which can lead to poor convergence. Secondly, its computation cost is significantly higher than the vanilla SAM, resulting in the increased training time (SAFER updates the selected layers periodically to avoid this dilemma, at the expense of losing the instant tracking of dynamic changes during training).

To address these challenges, we require a more robust mechanism that can balance exploration and exploitation. We therefore introduce an \textbf{adaptive sampling strategy}. Instead of adopting a fixed or greedy selection criterion, we maintain a probability distribution $p_t$ over the model's layers. This distribution is updated at each iteration based on the observed gradient norms, effectively allowing the algorithm to ``learn'' and track which layers are important over time in a stable manner. By framing the selection as a multi-armed bandit problem, we provide a principled way to manage the trade-off between exploration and exploitation, ensuring that all layers have a chance to be updated while still focusing computation on the most critical ones. This idea motivates our novel algorithm, \tool.

\subsection{\tool\ Algorithm}

Based on our motivation, we propose \textbf{Sparse-Layer SAM (\tool)}, a novel algorithm that improves the efficiency of SAM by introducing layer-wise sparsity. At each iteration $t$, \tool\ performs the following key steps: 
\begin{enumerate}
    \item \textbf{Layer Sampling:} It samples a subset of layers $S_t$ to be active based on the probability distribution $p_t$.
    \item \textbf{Sparse SAM Update:} It performs the two-step SAM update (perturbation and update) \textit{only} on the parameters of the active layers in $S_t$.
    \item \textbf{Distribution Update:} It updates the sampling distribution to $p_{t+1}$ for the next iteration based on the gradient norms obtained in the perturbation step.
\end{enumerate}
This process is illustrated in Figure \ref{fig1} and detailed in Algorithm \ref{alg1}. Next, we introduce the details of these steps.

\myparagraph{Step \ding{182}: Layer Sampling.}
We consider a model with $N$ layers. We maintain a probability distribution $p_t \in \mathbb{R}^N$, where the $l$-th element $p_{t,l}$ is the probability of selecting layer $l$ at iteration $t$. At the start of the iteration, we sample an active set of layers $S_t \subseteq [N]$ by performing a Bernoulli trial for each layer $l$ with success probability $p_{t,l}$. The initial distribution $p_1$ is uniform, with $p_{1,l} = s/N$ for all layers, where $s$ is a hyperparameter satisfies $\sum_{l} p_{t,l}=s$, representing the desired number of active layers.

\myparagraph{Step \ding{183}: Sparse SAM Update.}
For the sampled active layers $S_t$, we perform a two-step SAM update. First, we compute the stochastic gradient $r_{t,l} = \nabla f_l(x_t, \xi_t)$ for each layer $l \in S_t$. This gradient is used to obtain the perturbation $\epsilon_{t,l}$ with radius $\rho$. Second, we compute the gradient $g_{t,l}$ at the perturbed point and use it to update the parameters $x_{t,l}$. For this step, we use the AdamW optimizer \cite{loshchilov2017decoupled}, a standard choice for fine-tuning Transformer-based models. The parameters of all other layers ($l \notin S_t$) remain frozen during this iteration.

\myparagraph{Step \ding{184}: Adaptive Distribution Update.}
After the update, we adapt the sampling distribution for the next iteration. Our goal is to increase the sampling probability for layers with larger gradient norms. Motivated by \cite{liu2022communication}, we frame this as a multi-armed bandit problem and use an EXP3-based algorithm \cite{auer2002finite} to update the probabilities, as detailed in Algorithm \ref{alg2}. $p_{min}$ and $\alpha_p$ are hyper-parameters where $p_{min}$ denotes the lower bound of $p_{t,d}$'s. $D_{kl}(\cdot||\cdot)$ represents the KL-divergence. The layer-wise gradient norms $\|r_{t,l}\|$ from the active set are used to update the distribution $p_t$ to $p_{t+1}$, ensuring that layers which are currently more important for optimization are more likely to be selected in the future.

\begin{algorithm}[t]
  \caption{Sparse Layer SAM (\tool)}\label{alg1}
  \begin{algorithmic}[1]
  \REQUIRE Initial values $x_1$, perturbation radius $\rho$, learning rate $\eta$, $p_{1,l}=s/N, m_{0,l} = 0, v_{0,l} = \mathbf{0}, \forall l \in [N]$, coefficients $\beta_1, \beta_2 < 1$.
  \FOR{$t = 1, ..., T$}
    \STATE Sample active layers $S_t=\{l:Z_{t,l}=1, \forall l \in [N]\}$ according to the distribution $\mathds{P}(Z_{t,l}=1)=p_{t,l}$, where each $Z_{t,l}$ follows the Bernoulli distribution;
    \STATE Sample a minibatch $\xi_t$ from the dataset;
    \FOR{$l \in S_t$}
    \STATE Compute the gradient $r_{t,l}=\nabla f_l(x_t,\xi_t)$;
    \STATE Compute the perturbation $\epsilon_{t,l}=\rho \frac{r_{t,l}}{\|r_{t,l}\|}$;
    \STATE Compute the gradient $g_{t,l} = \nabla f_l(x_t+\epsilon_t,\xi_t)$;
    \STATE $x_{t+1,l} = \text{AdamW}(x_{t,l}, g_{t,l})$;
    \ENDFOR
    \STATE $p_{t+1,l} = \text{Update Distribution}(p_{t,l}, S_t, \|r_{t,l}\|)$;
  \ENDFOR
  \STATE
  \STATE \textbf{optimizer} AdamW($x_{t,l}, g_{t,l}$)
  \FOR{$l \in S_t$}
  \STATE $m_{t,l} = \beta_1 m_{t-1,l} + (1-\beta_1) g_{t,l}$;
  \STATE $v_{t,d} = \beta_2 v_{t-1,l} + (1-\beta_2) g_{t,l} \odot g_{t,l}$;
  \STATE $x_{t+1,l} = x_{t,l} - \eta \frac{m_{t,l}}{\sqrt{v_{t,l}+\mathbf{\epsilon}}} - \eta \lambda_t x_{t,l}$;
  \ENDFOR
  \STATE \textbf{end optimizer}
  \end{algorithmic}
\end{algorithm}

\begin{algorithm}[t]
  \caption{Update Distribution($p_{t,l}, S_t, \|r_{t,l}\|$)}\label{alg2}
  \begin{algorithmic}[1]
  \REQUIRE $p_{t,l}$ for $l \in [N]$, $S_t$, $\|r_{t,l}\|$ for $l \in S_t$.
  \FOR{$l \in [N]$}
    \IF{$l \in S_t$}
    \STATE $G$ is the maximal value of elements in $\|r_{t,l}\|$'s;
    \STATE $\Tilde{k}_{t,l} = -\frac{\|r_{t,l}\|^2}{p_{t,l}^2} + \frac{G^2}{p_{min}^2}$;
    \ELSE
    \STATE $\Tilde{k}_{t,l} =0$;
    \ENDIF
    \STATE $u_{t,l} = p_{t,l} \exp{(-\alpha_p \Tilde{k}_{t,l} /p_{t,l})}$;
    \STATE $p_{t+1,l} = \textit{argmin}_{q\in \mathcal{P}} D_{kl}(q||u_{t,l})$;
  \ENDFOR
  \ENSURE $p_{t+1,l}$
  \end{algorithmic}
\end{algorithm}

\myparagraph{Comparison with Prior Works:}
Our work is distinguished from prior art in efficient SAM optimization through its unique approach to adaptive, layer-wise sparsity in gradient computation. We categorize related works into two main groups: those that reduce the frequency of SAM updates and those that introduce sparsity like us.

A primary line of research reduces SAM's cost by decreasing its frequency of computing the perturbation gradients, either randomly (RST \cite{zhao2022randomized}), periodically (LookSAM \cite{liu2022towards}, SALA \cite{tan2024sharpness}), or only in the late stages of training \cite{zhou2024sharpness}. These methods reduce computation along the iteration axis. In contrast, our \tool\ reduces computation along the parameter axis at every single step. These two research directions are orthogonal and could be synergistically combined in future work.

Among sparsity-based methods, ESAM \cite{du2021efficient} differs from \tool\ in three crucial ways: (1) ESAM samples parameters randomly, whereas \tool\ uses an adaptive, gradient-norm-based strategy to focus on the most critical layers; (2) ESAM requires an extra forward pass for sample selection, adding overhead that we avoid; and (3) ESAM reverts to a dense gradient for the final update step, while \tool\ maintains sparsity throughout the entire process. Similarly, while SSAMF \cite{mi2022make} also adopts the principle of sparse perturbation, its implementation paradoxically incurs a higher computational cost than standard SAM as it periodically updates the mask and calculates the full-dimensional gradient before applying the mask. In contrast, \tool\ could provide a substantial reduction in gradient computation cost.

\subsection{Theoretical Analysis}
To provide theoretical guarantees for our algorithm, we make the following standard assumptions, which are common in the analysis for algorithms involving adaptive optimizers \cite{das2024towards,zhou2024towards,li2025frac}.

\begin{assumption}[Coordinate-wise $L$-smoothness]\label{assu1}
    $f(x,\xi)$ is differentiable and satisfies the following inequality:
    \begin{equation}
        \|\nabla f_i(x,\xi) - \nabla f_i(y,\xi)\| \leq L \|x_i - y_i\|, \forall i \in [d]. \nonumber
    \end{equation}
\end{assumption}

\begin{assumption}[Coordinate-wise Gradient Variance Bounded]\label{assu2}
    There exist positive constants $\sigma_i$ such that
    \begin{equation}
        \mathbb{E}\|\nabla f_i(x,\xi) - \nabla f_i(x)\|^2 \leq \sigma_i^2, \forall i \in [d].\nonumber
    \end{equation}
    We denote $\sigma_s^2 = \sum_i \sigma_i^2$ additionally.
    
\end{assumption}

\begin{theorem}
    If $f(x)$ in Algorithm \ref{alg1} satisfies Assumptions \ref{assu1} and \ref{assu2}. Assume the constant $\gamma \in (0,1]$, denote $\Hat{\sigma}^2 = \max\{4\sigma_s^2 + 12d\rho^2 L^2, \frac{L(f(x_1)-f^*)}{\gamma^2 T}\}$. We set the coefficients satisfy $1-\sqrt{\beta_1} = \sqrt{\frac{L(f(x_1)-f^*)}{\Hat{\sigma}^2 T}}$, $\beta_2 \leq \sqrt{\beta_1}$, and the parameters satisfy $\epsilon = \frac{\Hat{\sigma}^2}{d}$, $\eta=\mathcal{O}(\sqrt{\frac{1}{dT}})$, $\lambda = \mathcal{O}(\frac{\sqrt{d}}{(T^3 \Hat{\sigma}^2)^{1/4}})$. Let the model is initialized with $\|x_1\|_\infty = \mathcal{O}(\frac{T}{d})$. Then we have the following result for Algorithm \ref{alg1}
    \begin{equation}
        \frac{1}{T}\sum_{t=1}^T \mathbb{E}\|\nabla f(x_t)\|_1 = \mathcal{O}\bigg(\frac{\sqrt{d}}{T^{1/4}} + \sqrt{\frac{d}{T}}\bigg).
    \end{equation}
\end{theorem}
\begin{corollary}
We can conclude the $\frac{1}{T}\sum_{t=1}^T \mathbb{E}\|\nabla f(x_t)\|_1 = \mathcal{O}(T^{-1/4})$ convergence rate for SL-SAM from the above theorem. This rate recovers the theoretical result in other SAM-related works \cite{andriushchenko2022towards,mi2022make}.
\end{corollary}

\begin{remark}
The proof for our theorem is a corollary of the conclusion in \cite{li2025frac}. Note that this proof does not rely on the assumption of gradient bounded adopted in \cite{liu2022communication,sun2024adasam,zhou2024towards}. We provide a detailed derivation in the supplementary material.
\end{remark}

\begin{table*}[!t]
    \renewcommand{\arraystretch}{1.3}
    \caption{Experimental results of fine-tuning DeiT on CIFAR 10 and CIFAR100.}
    \label{table1}
    \centering
    \begin{tabular}{cccccll}
    \toprule
    \multirow{2}{*}{Algorithm} & \multicolumn{2}{c}{CIFAR-10} & \multicolumn{2}{c}{CIFAR-100} & \multirow{2}{*}{Max Mem(MB)} & \multirow{2}{*}{Epoch Time(s)}\\
    \cmidrule(r){2-3} \cmidrule(r){4-5} 
    & Accuracy(\%) &  Active Ratio & Accuracy(\%) &  Active Ratio & & \\
    \midrule
    AdamW & 98.04$\pm$0.02 &  $1.0\times$  & 87.49$\pm$0.08 &  $1.0\times$ & 8584 & 118.51 \\
    AdaSAM & 98.22$\pm$0.06 & $2.0\times$  & 87.86$\pm$0.06 & $2.0\times$ & 8546(100\%) & 231.39(100\%) \\
    RST & 98.16$\pm$0.07 & $1.5\times$  & 87.75$\pm$0.07 & $1.5\times$ & 8548(+0.02\%) & 179.70(-22.33\%) \\
    ESAM & 98.17$\pm$0.04 & $1.5\times$  & 87.72$\pm$0.03 & $1.5\times$ & 8188(-4.19\%) & 202.21(-12.61\%) \\
    SSAM-F & 98.08$\pm$0.08 & $2.0\times$  & 87.72$\pm$0.06 & $2.0\times$ & 8754(+2.43\%) & 242.11(+4.63\%) \\
    \midrule
    \rowcolor{gray!15}
    \tool\ (ours) & 98.12$\pm$0.04 & $0.942\times$  & 87.83$\pm$0.08 & $0.961\times$ & 7832(-8.35\%) & 198.35(-14.28\%) \\
    \bottomrule
    \end{tabular}
\end{table*}

\section{Experiments}
We evaluate \tool\ by fine-tuning models of varying scales on several benchmark tasks\footnote{Codes are available at \texttt{https://github.com/chengyif/SL-SAM.}}. We compare it against a suite of strong baselines: \textbf{AdamW} \cite{loshchilov2017decoupled}, \textbf{AdaSAM} \cite{sun2024adasam} (Adaptive base optimizer version of SAM) and three efficient SAM variants \textbf{RST} \cite{zhao2022randomized}, \textbf{ESAM} \cite{du2021efficient} and \textbf{SSAM-F} \cite{mi2022make}. All SAM-based optimizers use AdamW as their base optimizer since it is the default for transformer models. Experiments are conducted on a single NVIDIA RTX A6000 GPU or A100 GPU, where we measure both task performance and efficiency in terms of time and memory. To ensure broader applicability, we adapted the baseline to several popular frameworks. We note that this adaptation process may introduce some fluctuations in memory cost, as their original implementations were designed for their initial network architecture.

\subsection{Fine-tune Vision-Transformer Model}

\noindent\textbf{Models and Dataset.} We conduct experiments on a Vision Transformer model: DeiT-Small (21.7M parameters)~\cite{touvron2021training}. We fine-tune this model on the CIFAR10 and CIFAR100 datasets, initializing it from publicly available checkpoints pre-trained on ImageNet-1k.

\noindent\textbf{Common Hyper-parameters.} For all experiments, we use a batch size of 128, an initial learning rate of $10^{-4}$ according to \cite{li2024friendly} and decayed with the cosine schedule. The weight decay is $5\times 10^{-5}$ and the model is fine-tuned for 20 epochs. 

\noindent\textbf{Method-Specific Hyper-parameters.} For all SAM-based methods, we set the perturbation radius $\rho=0.01$, following prior work~\cite{sun2024adasam}. For the baselines, we adopt the hyper-parameters recommended in their respective papers: the probability of performing the perturbation step for RST and the sparsity for SSAM-F are set to 0.5; the selecting ratios of parameters and samples for ESAM are also both set to 0.5. For our method, we tune the layer sparsity parameter $s/N$ and find the value 0.2 to be optimal.

\begin{figure}[t]
\centering
\includegraphics[width=0.48\textwidth]{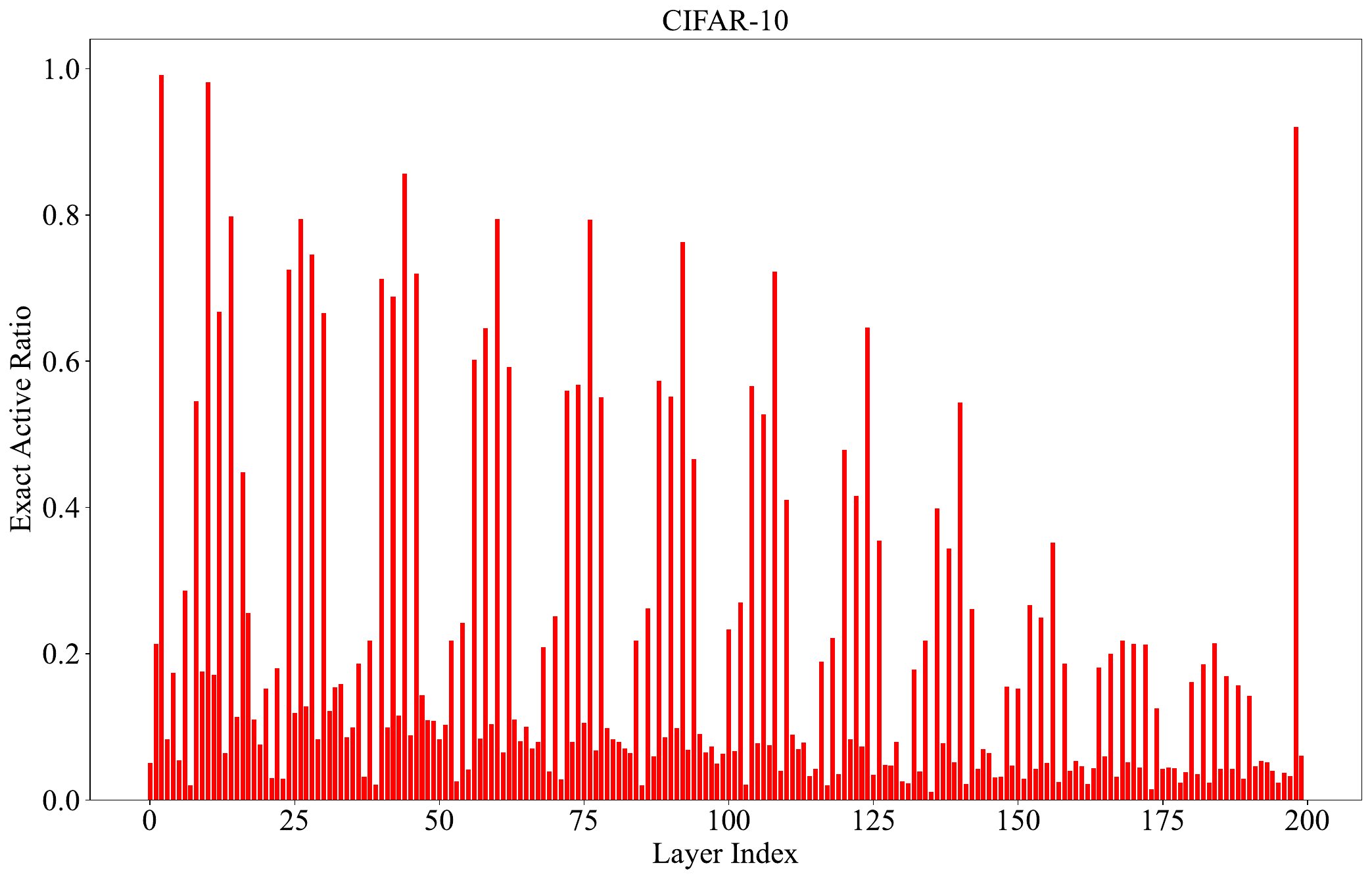}
\caption{The exact active ratio for each layer throughout the DeiT-Small fine-tuning on CIFAR-10.}
\label{count}
\end{figure}

\noindent\textbf{Performance Analysis.}
We report the results of the DeiT fine-tuning experiments over three runs with different random seeds in Table~\ref{table1}. ``Accuracy'' represents the best test accuracy in all epochs.
Across both CIFAR-10 and CIFAR-100 datasets, AdaSAM achieves the highest accuracies. \tool\ shows a negligible performance drop compared to AdaSAM, and consistently performs on par with the other SAM-based baselines, especially outperforms them on CIFAR-100. This demonstrates that our adaptive layer selection strategy successfully preserves the generalization benefits of SAM, even while a significant portion of the model is frozen in each step. While full-dimensional gradient methods like AdaSAM is at the top rank, they do so at a prohibitive cost, which \tool\ is designed to solve.

\noindent\textbf{Computational and Memory Efficiency.}
Though the parameter ``$s/N$" in \tool\ represents the ratio of layers participating in backpropagation, the number of model parameters varies across all layers. 
We count the actual number of parameters that participate in the gradient calculation steps and report its ratio to that of AdamW as ``Active Ratio". The values of RST and ESAM are evaluated according to 0.5 sparsity since they use uniform sampling. \tool\ activates only \textbf{47.1\%} of the parameters per gradient pass in average on CIFAR10. Since SAM requires two backward passes, \tool's total computation per step involves 94.2\% of the model parameters (2$\times$47.1\%). In the same way, 96.1\% parameters participate in gradient computation on CIFAR-100. These are dramatically lower than the 200\% required by AdaSAM and SSAM-F, and also lower than the 100\% of an AdamW step.

This computational efficiency directly translates to significant memory and time savings. As shown in Table~\ref{table1}, by caching fewer gradient coordinates for backpropagation, \tool\ consumes \textbf{8.35\% less GPU memory} than AdaSAM on CIFAR10.\footnote{Max Mem. measures the peak memory usage throughout the training obtained with Nvidia's \texttt{nvidia-smi} command. Considering the evaluation on the test set would influence the memory usage, ``Max Mem" and ``Epoch Time" are counted in the run without evaluation.} Benefiting from the calculation of sparse gradient, \tool\ reduces the average epoch time by \textbf{14.28\%} compared to AdaSAM. Though the time reduction of RST is larger than \tool\, it consumes the same GPU memory as AdaSAM. The savings of ESAM are both weaker than \tool. These results confirm that \tool\ is not only effective but also highly efficient, making it well-suited for memory-constrained environments or for scaling up batch sizes to accelerate the entire training process. 

In addition, for each layer, we count the exact ratio of iterations it is activated to participate in gradient computation in the experiment on CIFAR-10. We plot the result in Figure \ref{count}. Under the setting of sparsity parameter ``$s/N$'' as 0.2, only 15\% of model layers are active in more than half of the iterations, and none of layers are with extremely low active ratio. This indicates that our strategy guarantees sufficient training for all layers, achieving the trade-off between exploration and exploitation successfully.

\begin{table*}[!t]
    \renewcommand{\arraystretch}{1.3}
    \caption{\textbf{Experimental results of fine-tuning RoBERTa model on GLUE.}  \textbf{Metrics:} We report Matthew's correlation for CoLA, Pearson correlation for STS-B, F1 score for MRPC, and accuracy for the remaining tasks. \tool\ achieves the top score on 1/8 tasks, and the second score on 2/8 tasks. In addition, the average score of \tool\ is ranked third.}
    \label{table2}
    \centering
    \begin{tabular}{c|cccccccc|c|c}
    \toprule
    Algorithms & CoLA & STS-B & MRPC & RTE & SST2 & MNLI & QNLI & QQP & Average & Active Ratio\\
    \midrule
    AdamW & 61.64 & 90.91 & 91.68 & 77.98 & 94.04 & 87.43 & 92.75 & 91.90 & 86.04 & $1\times$\\
    AdaSAM & 63.08 & 90.92 & 92.50 & 80.87 & 95.41 & 87.46 & 92.99 & 92.12 & 86.92 & $2\times$\\
    RST & 62.38 & 90.97 & 92.39 & 79.42 & 95.07 & 87.57 & 93.25 & 91.89 &  86.62 & $1.5\times$\\
    ESAM & 63.41 & 90.96 & 92.88 & 80.51 & 94.50 & 87.52 & 92.90 & 91.97 & 86.83 & $1.5\times$\\
    SSAM-F & 61.32 & 91.11 & 92.73 & 80.14 & 95.53 & 87.59 & 93.01 & 92.16 & 86.70 & $2\times$\\
    \midrule
    \rowcolor{gray!15}
    \tool\ (ours) & 63.36 & 90.38 & 93.45 & 79.78 & 95.41 & 87.55 & 92.64 & 91.63 & 86.78 & $0.441\times$\\
    \bottomrule
    \end{tabular}
\end{table*}

\begin{table}[!t]
    \renewcommand{\arraystretch}{1.3}
    \caption{\textbf{GPU usage and average epoch time of SAM-type algorithms in two tasks.} Values in parentheses denote the percentage change relative to AdaSAM.}
    \label{table4}
    \centering
    \begin{tabular}{cll}
         \toprule
         Algorithm & Max Mem(MB) & Epoch Time(s) \\
         \midrule
         & \makecell[c]{CoLA task} & \\
         \midrule 
         AdaSAM & 3296($100\%$) & 52.48($100\%$) \\
         RST & 3316($+0.61\%$) & 40.63($-22.58\%$) \\
         ESAM & 3160($-4.13\%$) & 54.23($+3.33\%$) \\
         SSAM-F & 7288($+121.12\%$) & 65.66($+25.11\%$) \\
         \midrule
         \rowcolor{gray!15}
         \tool\ & \textbf{3080($-6.55\%$)} & 43.66($-16.81\%$) \\
         \midrule
        & \makecell[c]{MNLI task} & \\
         \midrule
         AdaSAM & 8076($100\%$) & 3015.16($100\%$) \\
         RST & 8078($+0.02\%$) & 2314.12($-23.25\%$) \\
         ESAM &7336($-9.16$) & 2871.64($-4.76\%$) \\
         SSAM-F & 7426($-8.05\%$) & 3078.06($+2.09\%$) \\
         \midrule
         \rowcolor{gray!15}
         \tool\ & \textbf{6462}($-19.99\%$) & \textbf{2238.76}($-25.75\%$) \\
         \bottomrule
    \end{tabular}
\end{table}

\begin{table*}[htbp]
    \renewcommand{\arraystretch}{1.3}
    \caption{Performances on downstream tasks of fine-tuning Llama3-3B-Instruct on Open-Platypus. \tool\ achieves the top score on 2/8 tasks (bold), and the second score on 3/8 tasks (underlined). The average score of \tool\ is the highest.}
    \label{table5}
    \centering
    \begin{tabular}{c|cccccccc|c}
    \toprule
    Algorithms & STEM & Humanities & S-Sciences & RTE & Hellaswag & Boolq & Openbookqa & Arc-c & Average \\
    \midrule
    AdamW & 49.92 & 56.49 & 65.42 & 70.76 & 70.27 & 79.14 & 36.20 & 44.11 & 59.039 \\
    AdaSAM & 50.49 & 55.90 & 65.52 & 72.56 & 70.17 & 79.20 & 36.20 & 44.28 & 59.290 \\
    RST & 50.52 & 55.98 & 65.58 & 72.56 & 70.33 & 79.24 & 36.20 & 43.94 & 59.294 \\
    ESAM & 50.02 & 56.13 & 64.64 & 71.48 & 69.99 & 79.79 & 38.20 & 44.11 & 59.295 \\
    \midrule
    \rowcolor{gray!15}
    \tool\ (ours) & 49.83 & 55.73 & 65.00 & \textbf{73.29} & \textbf{70.33} & \underline{79.39} & \underline{37.20} & \underline{44.11} & \textbf{59.360} \\
    \bottomrule
    \end{tabular}
\end{table*}

\begin{table}[htbp]
    \renewcommand{\arraystretch}{1.3}
    \caption{GPU memory and epoch time of SAM-type algorithms in fine-tuning Llama3-3B-Instruct.}
    \label{table6}
    \centering
    \begin{tabular}{cll}
         \toprule
         Algorithm & Max Mem(MB) & Epoch Time(s) \\
         \midrule 
         AdaSAM & 51294($100\%$) & 15969($100\%$) \\
         RST & 51294($+0.00\%$) & 14317($-10.35\%$) \\
         ESAM & 48310($-5.82\%$) & 14013($-12.25\%$) \\
         \midrule
         \rowcolor{gray!15}
         \tool\ & \textbf{47388($-7.61\%$)} & \textbf{12690($-20.53\%$)} \\
         \bottomrule
    \end{tabular}
\end{table}

\subsection{Fine-tune Moderate Language Model}
\noindent\textbf{Model and Task.}
We next evaluate our method on natural language tasks by fine-tuning the \textbf{RoBERTa} model~\cite{liu2019roberta} (125M parameters) on the GLUE benchmark~\cite{wang2018glue}.

\noindent\textbf{Hyper-parameters.}
All methods are run for 30 epochs with a batch size of 16. To ensure a robust and fair comparison, we adopt the task-specific learning rates from the recent state-of-the-art work~\cite{zhao2024galore}. The weight decay is set to 0.01. All other hyper-parameters are kept consistent with our DeiT experiments: the SAM perturbation radius is $\rho=0.01$, and the method-specific settings for RST, ESAM, SSAM-F, and our \tool\ remain unchanged (except that $s/N=0.25$ on the RTE dataset).

\noindent\textbf{Performance Analysis.}
As shown in Table~\ref{table2}, \tool\ secures the highest score on one task (MRPC), second-highest on two tasks (CoLA and SST2), and achieves the third average performance across all baselines. 

AdaSAM achieves the best performance benefiting from the exact gradients that come at a high cost as discussed below. The gap between \tool\ and AdaSAM is tiny, which illustrates that adopting sparse gradients suffices to perform well in this experiment. The performance of \tool\ is comparable to ESAM and SSAM-F. These results verify that SAM finds a flatter loss landscape compared to AdamW, which is in line with the experimental results in \cite{bahri2021sharpness,zhong2022improving,sun2024adasam}.

\noindent\textbf{Computational and Memory Efficiency.}
The great performance of \tool\ is achieved with remarkable efficiency. Due to RoBERTa's relatively uniform layer sizes, our sparsity setting of $s/N=0.2$ results in activating only 22\% of parameters per pass. This means the total computational cost per step (2 passes $\times$ 22\% $\approx$ 44\%) is far less than a standard AdamW step (100\%) and dramatically lower than the 200\% of full-dimensional SAM methods. 

This fundamental efficiency leads to the lowest time and memory cost. We select two tasks with different levels of data size, and count the maximal GPU memory usage and the average wall-clock time of one epoch for the SAM-type algorithms, as detailed in Table~\ref{table4}: the GPU memory cost of \tool\ is 6.55\% and 19.99\% lower than AdaSAM on CoLA and MNLI tasks, and the epoch time is 16.81\% and 25.75\% lower. We emphasize that our algorithm is the only method that achieves both the reduction of memory and time across two tasks, and the magnitude of the reduction is also the greatest in three of four metrics. Though RST has a significant time saving on CoLA, it still requires high memory for its full backpropagation steps. ESAM reduces the GPU memory on both tasks, but the complex sampling overheads diminish the speed gains. As a comparison, our bandit-based layer selection is lightweight and highly effective. Thus, in summary, \tool\ achieves superior generalization while simultaneously contributing to significant savings for time and memory efficiency among SAM-based optimizers.

\subsection{Fine-tune Large Language Model}
\noindent\textbf{Model and Task.}
Inspired by recent findings that SAM can mitigate catastrophic forgetting during continual fine-tuning~\cite{li2024revisiting}, we simulate a similar scenario by fine-tuning the \texttt{Llama-3.2-3B-Instruct} model. As the model has been instruction fine-tuned, we continue to fine-tune the checkpoint on the Open-Platypus dataset \cite{lee2023platypus} that aims to improve the reasoning ability of the model. We evaluate the final model checkpoint on a diverse suite of downstream benchmarks to measure its generalization across different cognitive abilities. Our evaluation tasks include:
\begin{itemize}
    \item \textbf{Domain Knowledge:} MMLU (subsets: STEM, Humanities, Social Sciences)~\cite{hendrycks2020measuring}.
    \item \textbf{Commonsense Reasoning:} RTE~\cite{giampiccolo2007third}, Hellaswag~\cite{zellers2019hellaswag} and BoolQ~\cite{clark2019boolq}.
    \item \textbf{Question Answering:} OpenBookQA~\cite{mihaylov2018can}.
    \item \textbf{Science Exams:} ARC-Challenge~\cite{clark2018think}.
\end{itemize}

\noindent\textbf{Hyper-parameters and Baselines.}
We fine-tune for one epoch using a learning rate of $1 \times 10^{-5}$ and a batch size of 2. 
We try to complete the task on one GPU. This setup ensures a fair and direct comparison, as the publicly available implementations for most baseline methods are designed for single-GPU execution. Due to the maximal GPU usage exceeding the memory of the A100 GPU, SSAM-F is excluded from this experiment. All other method-specific parameters remain consistent with the experiments on GLUE.

\noindent\textbf{Performance Analysis.}
The results on downstream tasks are summarized in Table~\ref{table5}. \tool\ achieves the highest score on two tasks (shown in bold in the Table) and the second highest score on three tasks (underlined). Especially, \tool\ achieves the \textbf{highest average score} across all downstream benchmarks, demonstrating that our adaptive sparsity indeed preserves the model generalization ability in complex and various tasks. The advantage over ESAM underscores the superiority of our bandit-based layer selection over simpler heuristics, particularly in this low batch size regime where ESAM's sample selection strategy offers little benefit. Finally, the performance gain of all SAM-based baselines over AdamW corroborates recent findings on the role of SAM in mitigating catastrophic forgetting~\cite{li2024revisiting}.

\noindent\textbf{Computational and Memory Efficiency.}
\tool's leadership is also reflected in its excellent efficiency which is a critical factor for LLM fine-tuning. We count that the exact sparsity of gradient in \tool\ is $21.03\%$, indicating the $42.06\%$ active ratio for gradient calculation, which is far lower than AdamW and other baselines. As shown in Table~\ref{table6}, it reduces peak GPU memory by \textbf{7.61\%} and wall-clock time by a significant \textbf{20.53\%} compared to AdaSAM. This advantage is a direct result of our adoption of a more aggressive strategy for saving computation cost. Meanwhile, the other baselines suffer from some inherent limitations: RST still incurs the full memory cost of the vanilla SAM, and ESAM only sparsifies the gradient ascent step while leaving the descent step fully dense. The reductions on epoch time of these two methods also have an obvious gap to \tool. By applying sparsification to \textit{both} the perturbation and update steps, \tool\ provides a more comprehensive and effective reduction in computational load. In the resource-intensive domain of LLM fine-tuning, where every percentage point of efficiency matters, \tool\ is validated as a powerful and uniquely practical solution.

\subsection{Robustness and Ablation Study}

\textbf{Robustness Study.} An important hyper-parameter of our proposed algorithm SL-SAM is $s/N$, i.e. the desired ratio of active model layers, also known as the sparsity of model layers, which is set to 0.2 in most experiments. In this section, we show the performances across different sparsities in fine-tuning the DeiT-Small model. As shown in Figure \ref{robust}, the test accuracies are close in the sparsity range of [0.2,0.25,0.3,0.35] on both CIFAR-10 and CIFAR-100 datasets and have a drop at 0.15. Though the average accuracy of three runs is at the highest level when the sparsity is set to 0.35, the performance at sparsity 0.2 has a negligible gap to it and obviously could save more GPU memory and time cost. To sum up, we adopt the selection of 0.2 in our experiments in the main body. 

\begin{figure}[t]
\centering
\includegraphics[width=0.48\textwidth]{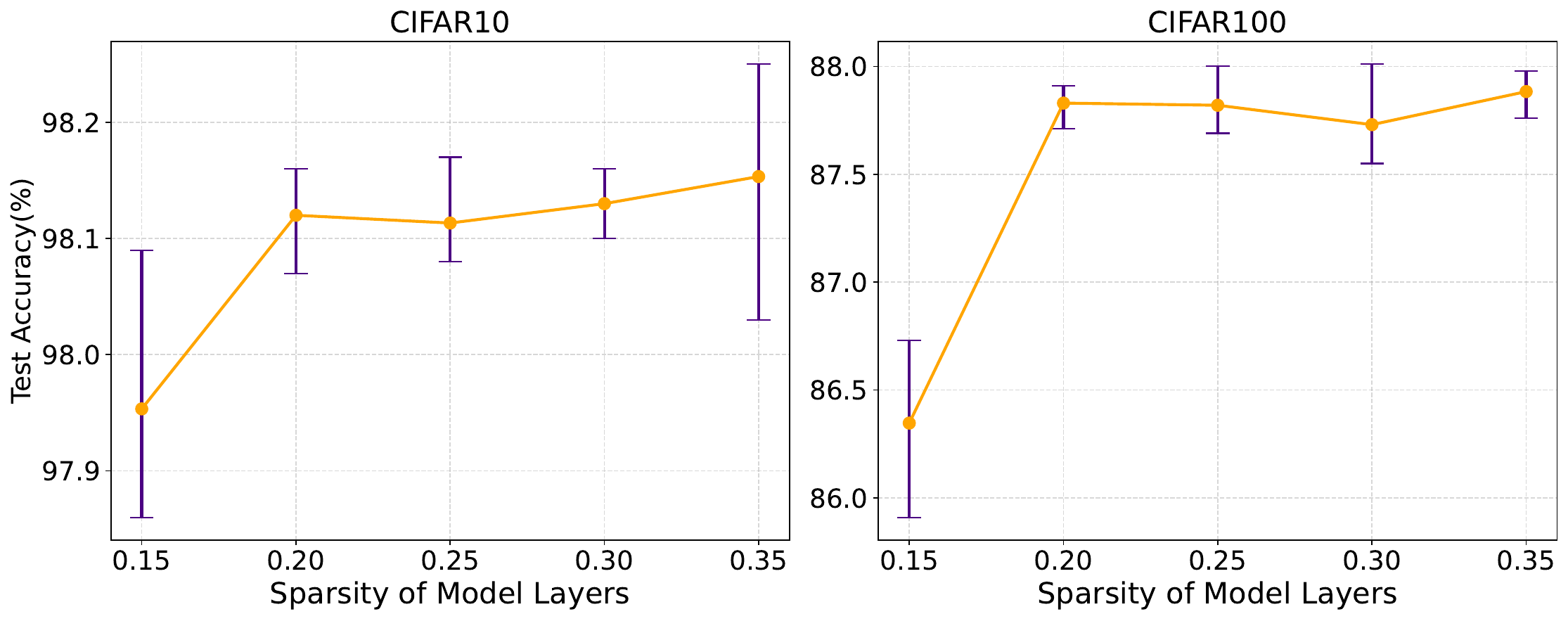}
\caption{\textbf{Test Accuracy v.s. Sparsity of Model Layers in SL-SAM.} Left: CIFAR-10 dataset; Right: CIFAR-100 dataset.}
\label{robust}
\end{figure}

\noindent\textbf{Ablation Study.} The core of SL-SAM is the identification of model layers through a multi-armed bandit method. We also compare this with two heuristic methods: the first one selects the layers randomly with a uniform distribution, named as Random-SL-SAM; the second one computes the full-dimensional gradient then selects the layers with the top-k largest norms in a greedy manner, named as Top-SL-SAM. We compare these three algorithms on fine-tuning DeiT-Small. Top-SL-SAM fails to converge just after the first epoch. We make a deep investigation into this result and find that as many as 137 of 200 model layers are not selected to participate in training throughout the whole epoch. This is in line with the insight in the ``Introduction" section: ``a greedy approach is susceptible to the stochasticity of mini-batch gradients, failing to explore layers that are temporarily dormant but may become critical for convergence". The test accuracies of Random-SL-SAM are $97.93\pm0.08\%$ and $86.19\pm0.21\%$ on CIFAR-10 and CIFAR-100 respectively, have a gap of \textbf{0.19\%} and \textbf{1.64\%} to SL-SAM, as well as a larger standard deviation. This advantage verifies that our strategy focuses the gradient calculation on those impactful layers, thus improving the convergence of the model.

\section{Extension: Couple with Single-Step SAM}
We notice that a recent work \cite{ji2024single} proposes an algorithm $\mathrm{S}^2$-SAM which performs the perturbation step using the gradient calculated in the last iteration. As a result, $\mathrm{S}^2$-SAM only needs to calculate the gradient once per iteration, reducing the computation overhead of SAM remarkably. In this section, we couple the ``Sparse Layer Selection" technique in \tool\ with this Single-Step SAM to form a novel algorithm named as SL-$\mathrm{S}^2$-SAM. Two algorithms are listed as Algorithms \ref{alg3} and \ref{alg4} (we also employ AdamW as the base optimizer for the model update step). We conduct experiments to show that SL-$\mathrm{S}^2$-SAM could further reduce the cost of $\mathrm{S}^2$-SAM with the help of sparse gradients.

\begin{algorithm}[t]
  \caption{$\mathrm{S}^2$-SAM}\label{alg3}
  \begin{algorithmic}[1]
  \REQUIRE Initial values $x_1$, perturbation radius $\rho$, learning rate $\eta$, $m_{0,l} = 0, v_{0,l} = \mathbf{0}, \forall l \in [N]$, coefficients $\beta_1, \beta_2 < 1$.
  \FOR{$t = 1, ..., T$}
    \STATE Sample a minibatch $\xi_t$ from the dataset;
    \STATE Compute the perturbation $\epsilon_t=\rho \frac{g_{t-1}}{\|g_{t-1}\|}$;
    \STATE Compute the gradient $g_t = \nabla f(x_t+\epsilon_t,\xi_t)$;
    \STATE $x_{t+1} = \text{AdamW}(x_t, g_t)$;
  \ENDFOR
  \end{algorithmic}
\end{algorithm}

\begin{algorithm}[t]
  \caption{SL-$\mathrm{S}^2$-SAM}\label{alg4}
  \begin{algorithmic}[1]
  \REQUIRE Initial values $x_1$, perturbation radius $\rho$, learning rate $\eta$, $p_{1,l}=s/N, m_{0,l} = 0, v_{0,l} = \mathbf{0}, \forall l \in [N]$, coefficients $\beta_1, \beta_2 < 1$.
  \FOR{$t = 1, ..., T$}
    \STATE Sample active layers $S_t=\{l:Z_{t,l}=1, \forall l \in [N]\}$ according to the distribution $\mathds{P}(Z_{t,l}=1)=p_{t,l}$, where each $Z_{t,l}$ follows the Bernoulli distribution;
    \STATE Sample a minibatch $\xi_t$ from the dataset;
    \FOR{$l \in S_t$}
    \STATE Compute the perturbation $\epsilon_{t,l}=\rho \frac{g_{t-1,l}}{\|g_{t-1,l}\|}$;
    \STATE Compute the gradient $g_{t,l} = \nabla f_l(x_t+\epsilon_t,\xi_t)$;
    \STATE $x_{t+1,l} = \text{AdamW}(x_{t,l}, g_{t,l})$;
    \ENDFOR
    \STATE $p_{t+1,l} = \text{Update Distribution}(p_{t,l}, S_t, \|g_{t,l}\|)$;
  \ENDFOR
  \end{algorithmic}
\end{algorithm}

\begin{table*}[!t]
    \renewcommand{\arraystretch}{1.3}
    \caption{Experimental results of fine-tuning DeiT on CIFAR 10 and CIFAR100.}
    \label{ext_table3}
    \centering
    \begin{tabular}{cccccll}
    \toprule
    \multirow{2}{*}{Algorithm} & \multicolumn{2}{c}{CIFAR-10} & \multicolumn{2}{c}{CIFAR-100} & \multirow{2}{*}{Max Mem(MB)} & \multirow{2}{*}{Epoch Time(s)}\\
    \cmidrule(r){2-3} \cmidrule(r){4-5} 
    & Accuracy(\%) &  Active Ratio & Accuracy(\%) &  Active Ratio & & \\
    \midrule
    $\mathrm{S}^2$-SAM & 98.10$\pm$0.08 &  $1.0\times$  & 87.59$\pm$0.07 &  $1.0\times$ & 8466(100\%) & 119.76(100\%) \\
    % \midrule
    SL-$\mathrm{S}^2$-SAM & 98.04$\pm$0.02 & $0.651\times$  & 87.45$\pm$0.05 & $0.676\times$ & 8216(-2.95\%) & 114.47(-4.42\%) \\
    \bottomrule
    \end{tabular}
\end{table*}

\begin{table*}[!t]
    \renewcommand{\arraystretch}{1.3}
    \caption{Experimental results of fine-tuning two language models: \textbf{(Top):} RoBERTa model on the GLUE benchmark; \textbf{(Bottom):} Downstream tasks performances after fine-tuning Llama-3.2-3B-Instruction on Open-Platypus.}
    \label{ext_table4}
    \centering
    \begin{tabular}{c|cccccccc|c|c|l}
    \toprule
    Algorithms & CoLA & STS-B & MRPC & RTE & SST2 & MNLI & QNLI & QQP & Average & Active Ratio & Epoch Time(s)\\
    \midrule
    $\mathrm{S}^2$-SAM & 58.51 & 90.86 & 92.69 & 79.78 & 95.07 & 87.25 & 92.90 & 91.95 & 86.13 & $1\times$ & 1630(100\%) \\
    SL-$\mathrm{S}^2$-SAM & \textbf{59.33} & 90.59 & \textbf{93.21} & 79.42 & 94.27 & \textbf{87.28} & 92.62 & 90.69 & 85.93 & $0.212\times$ & 1129(-30.74\%) \\
    \midrule
    \midrule
    Algorithms & STEM & Human. & S-Sciences & RTE & Hella. & Boolq & Open. & Arc-c & Average & Active Ratio & Epoch Time(s)\\
    \midrule
    $\mathrm{S}^2$-SAM & 50.14 & 56.58 & 65.19 & 70.76 & 70.14 & 79.08 & 36.60 & 44.11 & 59.075 & $1\times$ & 8425(100\%)\\
    SL-$\mathrm{S}^2$-SAM & 49.92 & 55.92 & 64.06 & \textbf{71.48} & 69.79 & 78.35 & \textbf{38.00} & 43.26 & 58.848 & $0.477\times$ & 7171(-14.9\%) \\
    \bottomrule
    \end{tabular}
\end{table*}

\noindent\textbf{Model Setups.} We conduct the same experiments as those in section IV: fine-tuning DeiT-Small on CIFAR, RoBERTa on GLUE and Llama-3.2-3B-Instruction models on Open-Platypus. The learning rate is tuned in [5e-5, 8e-5, 1e-4, 1.5e-4] when training DeiT-Small. The layer sparsity of SL-$\mathrm{S}^2$-SAM is set to 0.3 for DeiT-Small, 0.2 for all datasets in GLUE and 0.5 for Llama-3.2-3B-Instruction. All other hyper-parameters are kept unchanged. Specially, SL-$\mathrm{S}^2$-SAM runs the standard AdamW step for all layers in the first iteration to initial the record of the previous gradients.

\noindent\textbf{Performance Analysis.} We summarize the experimental results in Tables \ref{ext_table3} and \ref{ext_table4} respectively. On two CIFAR datasets, the test accuracies of SL-$\mathrm{S}^2$-SAM are marginally lower than those of $\mathrm{S}^2$-SAM. For the language models fine-tuning, SL-$\mathrm{S}^2$-SAM is ahead of $\mathrm{S}^2$-SAM in three of eight datasets on GLUE (shown in bold), as well as yielding higher scores than SL-$\mathrm{S}^2$-SAM in two tasks during the evaluations of Llama-3.2-3B-Instruction on downstream tasks. Across all experimental cases, the gap of average scores between SL-$\mathrm{S}^2$-SAM and $\mathrm{S}^2$-SAM is consistently within 0.2.

The performances of SL-$\mathrm{S}^2$-SAM are close to $\mathrm{S}^2$-SAM, but it is at a disadvantage overall. We analyze this phenomenon and think this is due to the latencies of the adopted previous gradients for perturbation of these two algorithms have an obvious difference. Specifically, $\mathrm{S}^2$-SAM uses the gradients in the last iteration for perturbation, which yields a latency of one iteration compared to the standard SAM. As a result, the performances of $\mathrm{S}^2$-SAM are falling behind two-step SAM methods (results shown in the last section) in all fine-tuning experiments. Then, since SL-$\mathrm{S}^2$-SAM samples partial layers to calculate the gradients in each iteration, the latency of the gradients for perturbation equals the interval between two sampled iterations for one model layer. This latency is no less than 1 and may be large for some seldom sampled layers, leading to the direction of perturbation has a bias to that of the current gradient adopted in standard SAM. Thus, we speculate that the latency of the gradient for perturbation is an important factor that affects the performance. This disadvantage also results in the variation of the value of sparsity across three sets of experiments. We consider to improve the performance SL-$\mathrm{S}^2$-SAM through correcting the perturbation direction in the future work.

\noindent\textbf{Computational and Memory Efficiency.} We count the actual active ratio of model parameters, peak GPU memory usage and average epoch training time in three experiments. The active ratio of parameters is 65.1\% in fine-tuning DeiT-Small on CIFAR10, yielding a \textbf{2.95\%} and \textbf{4.42\%} reduction on GPU memory and epoch time. Then on all datasets in GLUE, the average active ratio of SL-$\mathrm{S}^2$-SAM is 21.2\%, significantly lower than the 1.0 of $\mathrm{S}^2$-SAM. We select the MNLI dataset as a representative, and the run on it shows a memory reduction of \textbf{17.38\%} and a time cost reduction of \textbf{30.74\%}. Finally, SL-$\mathrm{S}^2$-SAM activates 47.7\% parameters in average to participate in the training on Llama-3.2-3B-Instruction, contributing to \textbf{14.9\%} of the training time saving. These results indicate that our layer selection strategy consistently help to improve the efficiency of single-step SAM by reducing the overhead of gradient calculation. 

\section{Conclusion}
To address the prohibitive computational cost of SAM, we introduced \tool\ that applies SAM's steps to a dynamically selected subset of layers. Grounded in the principle that the gradient norm is a strong proxy for layer importance for both of SAM's perturbation and update steps, \tool\ uses an adaptive sampling strategy to focus computation on impactful layers. We prove its convergence and empirically demonstrate that it matches the performance of vanilla SAM while substantially reducing computational overhead, resulting in great savings of memory and time in fine-tuning. The extension work couples the sparse gradient computation with the single-step SAM. The experiments consistently show a negligible performance drop and an obvious cost reduction. In summary, our work presents \tool\ as a practical algorithm that makes the SAM-based training computationally feasible for more machine learning applications.

% or


\appendices


\ifCLASSOPTIONcaptionsoff
  \newpage
\fi



\bibliographystyle{IEEEtran}
\bibliography{references}

\newpage
\onecolumn
\section{Detailed Proof of Theorem}

\begin{lemma}\label{lemma1}
    (Lemma 2 in \cite{li2025frac}) If $\beta_2 \leq \sqrt{\beta_1}$, then we have the following result for SL-SAM
    \begin{equation}
        \frac{m_{t,i}^2}{v_{t,i}} \leq 8, \forall t \in [T], i \in [d]. \nonumber
    \end{equation}
\end{lemma}

\begin{lemma}\label{lemma2}
    (Lemma 2 and ``Proof 1" in \cite{li2025frac}) If $\eta \lambda \leq \frac{\sqrt{v}}{5T^{5/4}}$, $\|x_1\|_\infty \leq \frac{\sqrt{v}}{4\lambda T^{1/4}}$, $\frac{\sqrt{v}}{T^{1/4}} < 1$, and $\beta_2 \leq \sqrt{\beta_1}$, then we have the following result for SL-SAM
    \begin{equation}
        \lambda \|x_t\|_\infty \leq \frac{\sqrt{v}}{T^{1/4}}, \forall t \in [T]. \nonumber
    \end{equation}
\end{lemma}

\begin{lemma}\label{lemma3}
If $f(x)$ in Algorithm 1 satisfies Assumptions \ref{assu1} and \ref{assu2}, then we have
\begin{equation}
    \mathbb{E}\|g_{t,i} - \nabla f_i(x_t)\|^2 \leq 12\rho^2 L^2 + 4\sigma_i^2, \quad \mathbb{E}\|g_{t,i}\|^2 \leq 12\rho^2 L^2 + 4\sigma_i^2 + 4\mathbb{E}\|\nabla f_i(x_t)\|^2. \nonumber
\end{equation}

\begin{proof}
\begin{align}
    &\mathbb{E}\|g_{t,i} - \nabla f_i(x_t)\|^2 \nonumber \\
    =& \mathbb{E} \|\nabla f_i(x_t+\rho \frac{\nabla f_i(x_t,\xi_t)}{\|\nabla f_i(x_t,\xi_t)\|}, \xi_t) - \nabla f_i(x_t+\rho \frac{\nabla f_i(x_t)}{\|\nabla f_i(x_t)\|}, \xi_t) + \nabla f_i(x_t+\rho \frac{\nabla f_i(x_t)}{\|\nabla f_i(x_t)\|}, \xi_t)  - \nabla f_i(x_t)\|^2 \nonumber \\
    \leq &2\rho^2 L^2 \mathbb{E}\|\frac{\nabla f_i(x_t,\xi_t)}{\|\nabla f_i(x_t,\xi_t)\|} - \frac{\nabla f_i(x_t)}{\|\nabla f_i(x_t)\|}\|^2 + 2\mathbb{E}\|\nabla f_i(x_t+\rho \frac{\nabla f_i(x_t)}{\|\nabla f_i(x_t)\|}, \xi_t) - \nabla f_i(x_t, \xi_t) + \nabla f_i(x_t, \xi_t) - \nabla f_i(x_t)\|^2 \nonumber \\
    \stackrel{(a)}{\leq} &8\rho^2 L^2 + 4\mathbb{E}\|\nabla f_i(x_t+\rho \frac{\nabla f_i(x_t)}{\|\nabla f_i(x_t)\|}, \xi_t) - \nabla f_i(x_t, \xi_t)\|^2 + 4\mathbb{E}\|\nabla f_i(x_t, \xi_t) - \nabla f_i(x_t)\|^2 \nonumber \\
    \stackrel{(b)}{\leq} & 12\rho^2 L^2 + 4\sigma_i^2, \nonumber
\end{align}
where (a) comes from Assumption \ref{assu1} and (b) comes from Assumptions \ref{assu1} and \ref{assu2}. The second inequality could be obtained directly in the same way.
\end{proof}
\end{lemma}

\begin{lemma}\label{lemma4}
If $f(x)$ in Algorithm 1 satisfies Assumptions 1 and 2, denote $\Tilde{v}_{t,i} = \beta_2 v_{t-1,i} + (1-\beta_2)(4|\nabla f_i(x_t)|^2 + 4\sigma_i^2 + 12\rho^2 L^2)$, denote $\|\sigma\|_1 = \sum_{i=1}^d \sigma_i$, then we have that
\begin{equation}
    \sum_{t=1}^T \sum_{i=1}^d \mathbb{E}\sqrt{\Tilde{v}_{t,i}+\epsilon} \leq 2\|\sigma\|_1 T + (4\rho L + \sqrt{\epsilon})d T + 8\sum_{t=1}^T \sum_{i=1}^d \mathbb{E}\frac{|\nabla f_i(x_t)|^2}{\sqrt{\Tilde{v}_{t,i} + \epsilon}}. \nonumber
\end{equation}
\end{lemma}
\begin{proof}
    Taking the result in Lemma \ref{lemma3} into the proof of Lemma 4 in \cite{li2025frac}, then obtaining the result.
\end{proof}

\begin{lemma}\label{lemma5}
If $f(x)$ in Algorithm 1 satisfies Assumption 1, denote $\theta = \lceil \frac{1}{1-\beta_1^{5/4}} \rceil$, then we have that
\begin{align}
    \mathbb{E}\|m_{t,i} - \nabla f_i(x_t)\|^2 &\leq \sqrt{\beta_1}\mathbb{E}\|m_{t-1,i} - \nabla f_i(x_{t-1})\|^2 + \frac{\beta_1^{3/4}\eta^2 L^2}{(1-\beta_1^{1/4})\sqrt{\epsilon}} \mathbb{E}\frac{\|m_{t-1,i}+\lambda x_{t-1,i} \sqrt{v_{t-1,i}+\epsilon}\|^2}{\sqrt{v_{t-1,i}+\epsilon}} \nonumber \\
    &+ \theta (1-\beta_1)^2 (4\sigma_i^2 + 12\rho^2 L^2) \nonumber
\end{align}

\end{lemma}
\begin{proof}
From the update rule of the algorithm, we have that
\begin{equation}
    m_{t,i} - \nabla f_i(x_t) = \beta_1 (m_{t-1,i} - \nabla f_i(x_{t-1})) - \beta_1 (\nabla f(x_t) - \nabla f(x_{t-1})) + (1-\beta_1) (g_{t,i} - \nabla f_i(x_t)) \nonumber
\end{equation}
Further, we have that
\begin{align}
&\mathbb{E}\|m_{t,i} - \nabla f_i(x_t)\|^2 \nonumber \\
\stackrel{(a)}{\leq} & \beta_1^{-5/4}\mathbb{E}\|\beta_1 (m_{t-1,i} - \nabla f_i(x_{t-1})) - \beta_1 (\nabla f_i(x_t) - \nabla f_i(x_{t-1}))\|^2 + \theta (1-\beta_1)^2 (4\sigma_i^2 + 12\rho^2 L^2) \nonumber \\
\stackrel{(b)}{\leq}& \beta_1^{3/4}(\beta_1^{-1/4}\mathbb{E}\|m_{t-1,i} - \nabla f_i(x_{t-1})\|^2 + \frac{1}{1-\beta_1^{1/4}}\mathbb{E}\|\nabla f_i(x_t) - \nabla f_i(x_{t-1})\|^2) + \theta (1-\beta_1)^2 (4\sigma_i^2 + 12\rho^2 L^2) \nonumber \\
\leq& \sqrt{\beta_1}\mathbb{E}\|m_{t-1,i} - \nabla f_i(x_{t-1})\|^2 + \frac{\beta_1^{3/4}\eta^2 L^2}{(1-\beta_1^{1/4})\sqrt{\epsilon}} \mathbb{E}\frac{\|m_{t-1,i}+\lambda x_{t-1,i} \sqrt{v_{t-1,i}+\epsilon}\|^2}{\sqrt{v_{t-1,i}+\epsilon}} + \theta (1-\beta_1)^2 (4\sigma_i^2 + 12\rho^2 L^2), \nonumber
\end{align}
where (a) and (b) come from Young's Inequality and Lemma \ref{lemma3}.
\end{proof}

\myparagraph{Proof of Theorem 1}
\begin{proof}
From the update rule of the algorithm and the $L$-smooth, we have that
\begin{align}\label{theo1}
    &f(x_{t+1}) - f(x_t) \nonumber \\
    \leq& \sum_{l \in S_t} \langle \nabla f_l(x_t), x_{t+1,l} - x_{t,l} \rangle + \frac{L}{2}\sum_{l \in S_t}\|x_{t+1,l} - x_{t,l}\|^2 \nonumber \\
    =& -\eta \sum_{l \in S_t} \langle \nabla f_l(x_t), \frac{m_{t,l}+\lambda x_{t,l}\sqrt{v_{t,l}+\epsilon}}{\sqrt{v_{t,l}+\epsilon}} \rangle + \frac{\eta^2 L}{2}\sum_{l \in S_t}\|\frac{m_{t,l}+\lambda x_{t,l}\sqrt{v_{t,l}+\epsilon}}{\sqrt{v_{t,l}+\epsilon}}\|^2 \nonumber \\
    \leq & -\frac{\eta}{2} \sum_{l \in S_t}\|\frac{\nabla f_l(x_t)}{(v_{t,l}+\epsilon)^{1/4}}\|^2 - \frac{\eta}{2}\sum_{l \in S_t}\|\frac{m_{t,l}+\lambda x_{t,l} \sqrt{v_{t,l}+\epsilon}}{(v_{t,l}+\epsilon)^{1/4}}\|^2 + \eta \sum_{l \in S_t} \|\frac{\nabla f_l(x_t) - m_{t,l}}{(v_{t,l}+\epsilon)^{1/4}}\|^2 + \eta \sum_{l \in S_t} \|\frac{\lambda x_{t,l} \sqrt{v_{t,l}+\epsilon}}{(v_{t,l}+\epsilon)^{1/4}}\|^2\nonumber \\
    &+ \frac{\eta^2 L}{2\sqrt{\epsilon}}\sum_{l \in S_t}\|\frac{m_{t,l} + \lambda x_{t,l} \sqrt{v_{t,l}+\epsilon}}{(v_{t,l}+\epsilon)^{1/4}}\|^2 \nonumber \\
    \leq & -\frac{\eta}{2} \sum_{l \in S_t}\|\frac{\nabla f_l(x_t)}{(v_{t,l}+\epsilon)^{1/4}}\|^2 - \frac{\eta}{4}\sum_{l \in S_t}\|\frac{m_{t,l}+\lambda x_{t,l} \sqrt{v_{t,l}+\epsilon}}{(v_{t,l}+\epsilon)^{1/4}}\|^2 + \eta \sum_{l \in S_t} \|\frac{\nabla f_l(x_t) - m_{t,l}}{\epsilon^{1/4}}\|^2 + \eta \sum_{l \in S_t} \|\frac{\lambda x_{t,l} \sqrt{v_{t,l}+\epsilon}}{(v_{t,l}+\epsilon)^{1/4}}\|^2
\end{align}
where the last inequality comes from $\eta \leq \frac{\sqrt{\epsilon}}{2L}$. Taking Lemma \ref{lemma2} into (\ref{theo1}) yields
\begin{align}\label{theo2}
    f(x_{t+1}) - f(x_t) \leq &-\frac{\eta}{2} \sum_{l \in S_t}\|\frac{\nabla f_l(x_t)}{(v_{t,l}+\epsilon)^{1/4}}\|^2 - \frac{\eta}{4}\sum_{l \in S_t}\|\frac{m_{t,l}+\lambda x_{t,l} \sqrt{v_{t,l}+\epsilon}}{(v_{t,l}+\epsilon)^{1/4}}\|^2 \nonumber \\
    &+ \eta \sum_{l \in S_t} \|\frac{\nabla f_l(x_t) - m_{t,l}}{\epsilon^{1/4}}\|^2 + \frac{\eta v}{\sqrt{T}}\sum_{l \in S_t} \|(v_{t,l}+\epsilon)^{1/4}\|^2
\end{align}
For the terms in the right hand of (\ref{theo2}), taking the expectation on the randomness in the layer sampling, we have that
\begin{align}\label{theo3}
    -\frac{\eta}{2} \mathbb{E}\sum_{l \in S_t}\|\frac{\nabla f_l(x_t)}{(v_{t,l}+\epsilon)^{1/4}}\|^2 \leq -\frac{\eta p_{min}}{2} \mathbb{E}\sum_{l \in S_t} \frac{1}{p_{t,l}}\|\frac{\nabla f_l(x_t)}{(v_{t,l}+\epsilon)^{1/4}}\|^2 = -\frac{\eta p_{min}}{2} \mathbb{E}\sum_{i=1}^d \frac{|\nabla f_i(x_t)|^2}{\sqrt{v_{t,i}+\epsilon}} 
\end{align}
Similarly, for other terms we have
\begin{align}\label{theo4}
&-\frac{\eta}{4}\mathbb{E}\sum_{l \in S_t}\|\frac{m_{t,l}+\lambda x_{t,l}\sqrt{v_{t,l}+\epsilon}}{(v_{t,l}+\epsilon)^{1/4}}\|^2 +\eta \mathbb{E} \sum_{l \in S_t} \|\frac{\nabla f_l(x_t) - m_{t,l}}{\epsilon^{1/4}}\|^2 + \frac{\eta v}{\sqrt{T}}\mathbb{E}\sum_{l \in S_t} \|(v_{t,l}+\epsilon)^{1/4}\|^2
 \nonumber \\
\leq & -\frac{\eta p_{min}}{4}\mathbb{E}\sum_{l \in S_t} \frac{1}{p_{t,l}}\|\frac{m_{t,l}+\lambda x_{t,l} \sqrt{v_{t,l}+\epsilon}}{(v_{t,l}+\epsilon)^{1/4}}\|^2 + \eta \mathbb{E} \sum_{l \in S_t} \frac{1}{p_{t,l}}\|\frac{\nabla f_l(x_t) - m_{t,l}}{\epsilon^{1/4}}\|^2 + \frac{\eta v}{\sqrt{T}}\mathbb{E}\sum_{l \in S_t} \frac{1}{p_{t,l}}\|(v_{t,l}+\epsilon)^{1/4}\|^2 \nonumber \\
=& -\frac{\eta p_{min}}{4}\mathbb{E}\sum_{i=1}^d \frac{|m_{t,i}+\lambda x_{t,i} \sqrt{v_{t,i}+\epsilon}|^2}{\sqrt{v_{t,i}+\epsilon}} + \eta \mathbb{E}\sum_{i=1}^d \frac{|f_i(x_t) - m_{t,i}|^2}{\sqrt{\epsilon}} + \frac{\eta v}{\sqrt{T}}\mathbb{E}\sum_{i=1}^d \sqrt{v_{t,i}+\epsilon} 
\end{align}
Substituting (\ref{theo3}) and (\ref{theo4}) into (\ref{theo2}) yields that 
\begin{align}\label{theo5}
    &\mathbb{E}[f(x_{t+1})]-f^* + \frac{\eta p_{min}}{4}\mathbb{E}\sum_{i=1}^d \frac{|m_{t,i}+\lambda x_{t,i} \sqrt{v_{t,i}+\epsilon}|^2}{\sqrt{v_{t,i}+\epsilon}} - \frac{\eta}{\sqrt{\epsilon}} \mathbb{E}\sum_{i=1}^d |f_i(x_t) - m_{t,i}|^2 \nonumber \\
    \leq & \mathbb{E}[f(x_t)]-f^* -\frac{\eta p_{min}}{2} \mathbb{E}\sum_{i=1}^d \frac{|\nabla f_i(x_t)|^2}{\sqrt{v_{t,i}+\epsilon}} + \frac{\eta v}{\sqrt{T}}\mathbb{E}\sum_{i=1}^d \sqrt{v_{t,i}+\epsilon}
\end{align}
Multiplying $\frac{\eta}{\sqrt{\epsilon}(1-\sqrt{\beta_1})}$ on both sides of Lemma \ref{lemma5}, then substituting the result into (\ref{theo5}) yields that
\begin{align}\label{theo6}
    &\mathbb{E}[f(x_{t+1})]-f^* + \frac{\eta p_{min}}{4}\mathbb{E}\sum_{i=1}^d \frac{|m_{t,i}+\lambda x_{t,i} \sqrt{v_{t,i}+\epsilon}|^2}{\sqrt{v_{t,i}+\epsilon}} + \frac{\eta \sqrt{\beta_1}}{(1-\sqrt{\beta_1})\sqrt{\epsilon}} \mathbb{E}\sum_{i=1}^d |f_i(x_t) - m_{t,i}|^2 \nonumber \\
    \leq & \mathbb{E}[f(x_t)]-f^* +  \frac{\eta p_{min}}{4}\mathbb{E}\sum_{i=1}^d \frac{|m_{t-1,i}+\lambda x_{t-1,i} \sqrt{v_{t-1,i}+\epsilon}|^2}{\sqrt{v_{t-1,i}+\epsilon}} + \frac{\eta \sqrt{\beta_1}}{(1-\sqrt{\beta_1})\sqrt{\epsilon}} \mathbb{E}\sum_{i=1}^d |f_i(x_{t-1}) - m_{t-1,i}|^2 \nonumber \\
    &-\frac{\eta p_{min}}{2} \mathbb{E}\sum_{i=1}^d \frac{|\nabla f_i(x_t)|^2}{\sqrt{v_{t,i}+\epsilon}} + \frac{\eta v}{\sqrt{T}}\mathbb{E}\sum_{i=1}^d \sqrt{v_{t,i}+\epsilon} + \frac{2\eta \theta(1-\beta_1)}{\sqrt{\epsilon}}(4\sigma_s^2+12d\rho^2 L^2),
\end{align}
where we use $\eta^2 \leq \frac{(1-\beta_1^{1/4})(1-\sqrt{\beta_1})}{4\beta_1^{3/4} L^2}$ in the last inequality. The inequality (\ref{theo6}) has been already in the similar format to the inequality (7) in ``Proof 1" of \cite{li2025frac}, except some terms scaled by the constant, but does not affect the convergence rate. The rest of the proof completely follows that proof, with the help of the above lemmas. The only difference is that we use $1-\sqrt{\beta_1}$ to replace $1-\theta$ in inequality (14) of \cite{li2025frac} to yield the value of $\beta_1$. 

Finally, we obtain that, denote $\hat{\sigma}^2 = \max\{4\sigma_s^2+12d\rho^2 L^2, \frac{L(f(x_1)-f^*)}{T\gamma^2}\}$ where $\gamma \in (0,1]$. Setting the coefficients as $1-\sqrt{\beta_1} = \sqrt{\frac{L(f(x_1)-f^*)}{\Hat{\sigma}^2 T}}$, $\beta_2 \leq \sqrt{\beta_1}$, $\epsilon = \frac{\Hat{\sigma}^2}{d}$, $\eta=\mathcal{O}(\sqrt{\frac{1}{dT}})$, $\lambda = \mathcal{O}(\frac{\sqrt{d}}{(T^3 \Hat{\sigma}^2)^{1/4}})$, initializing the model as $\|x_1\|_\infty = \mathcal{O}(\frac{T}{d})$, we could obtain the final convergence rate 
\begin{equation}
        \frac{1}{T}\sum_{t=1}^T \mathbb{E}\|\nabla f(x_t)\|_1 = \mathcal{O}\bigg(\frac{\sqrt{d}}{T^{1/4}} + \sqrt{\frac{d}{T}}\bigg).
    \end{equation}
\end{proof}

\end{document}